\documentclass[twoside, letterpaper]{article}

\usepackage[accepted]{aistats2020}

\usepackage[round]{natbib}

\usepackage[utf8]{inputenc} 
\usepackage[T1]{fontenc}    
\usepackage{hyperref}       
\usepackage{url}            
\usepackage{booktabs}       
\usepackage{amsfonts}       
\usepackage{nicefrac}       
\usepackage{microtype}      
\usepackage{floatrow}

\usepackage[pdftex]{graphicx}
\usepackage{amsmath,amsthm}
\usepackage{dsfont} 
\usepackage{natbib}
\usepackage{algorithm}
\usepackage{algorithmic}
\usepackage{booktabs}
\usepackage{IEEEtrantools}

\graphicspath{{Figures/}}

\newtheorem{theorem}{Theorem}
\newtheorem{definition}{Definition}
\newtheorem{corollary}{Corollary}

\newcommand*{\Scale}[2][4]{\scalebox{#1}{$#2$}}

\newcommand{\minus}{\scalebox{0.75}[1.0]{$-$}}
\newcommand{\pushcode}[1][1]{\hskip\dimexpr#1\algorithmicindent\relax}
\newcommand{\multicol}[2]{\IEEEeqnarraymulticol{#1}{#2}}

\begin{document}

\twocolumn[

\aistatstitle{Adaptive, Distribution-Free Prediction Intervals for Deep 
Networks}

\aistatsauthor{ Danijel Kivaranovic \And 
                Kory D. Johnson \And  
                Hannes Leeb }

\aistatsaddress{  ISOR \And  
                  Institute for Statistics and Mathematics\And 
                  ISOR, DS@UniVie } 

\vspace{-.75cm}

\aistatsaddress{  University of Vienna \And  
                  Vienna University of Economics and Business \And 
                  University of Vienna} 

]

\begin{abstract}
The machine learning literature contains several constructions for prediction
intervals that are intuitively reasonable but ultimately ad-hoc in that they 
do not come with provable performance guarantees. We present methods from the 
statistics literature that can be used efficiently with neural networks 
under minimal assumptions with guaranteed performance. We propose a neural 
network that outputs three values instead of a single point estimate and optimizes 
a loss function motivated by the standard quantile regression loss. We 
provide two prediction
interval methods with finite sample coverage guarantees solely under the assumption
that the observations are independent and identically distributed. The first 
method leverages the conformal inference framework and provides average coverage. 
The second method provides a new, stronger guarantee by conditioning on the
observed data.
Lastly, our loss function does not compromise the
predictive accuracy of the network like other prediction interval methods. We
demonstrate the ease of use of our procedures as well as its improvements over
other methods on both simulated and real data. As most deep networks can 
easily be modified by our method to output predictions with valid prediction 
intervals, its use
should become standard practice, much like reporting standard errors along with
mean estimates.
\end{abstract}

\section{Introduction}
Deep neural networks have gained tremendous popularity in the last decade due
to their superior predictive performance over other machine learning
algorithms. They have become the state-of-the-art algorithm in many challenging
tasks such as computer vision \citep{krizhevsky2012,karpathy2014}, speech 
recognition \citep{hinton2012}, natural language processing 
\citep{collobert2008}, and bioinformatics \citep{alipanahi2015}.
Despite these successes, there is a paucity of research on the uncertainty of
neural network predictions on new samples. 

The development of accurate prediction intervals (PIs) for neural networks is 
a challenging task that is only beginning to gain research interest.
Several authors have provided motivation for modified loss functions intended
to encourage desirable properties
\citep{nix1994,heskes1997,khosravi2011,LakPBak17,Pearce+18,KerenCS18,TagL18,
kuleshov2018}.
That being said, some provide a PI without a point estimate \citep{Pearce+18, 
KerenCS18} or use loss functions which cannot be optimized with stochastic 
gradient descent \citep{khosravi2011}. Others have distributional assumptions 
\citep{nix1994,LakPBak17} or provide intervals in which the lower 
bound is not guaranteed to be smaller than the upper bound \citep{TagL18}. These 
methods also do not provide rigorous guarantees that are desired for a PI. 
Uncertainty estimation of neural networks from Bayesian perspective has been 
analyzed by \citep{Mac92, GalG16}.


The statistics literature is full of confidence interval constructions with 
provable performance guarantees; however, these methods often 
cannot be used efficiently with modern learning machines. We use variations
of well-known statistical techniques to propose two neural-network-based 
PIs with valid coverage claims that outperform existing methods. The core of 
both procedures is a neural network that outputs 
three values and that optimizes a modified quantile regression loss function.
The first method is based on the conformal method introduced by \citep{Vovk05}.
\citet{Pap11} applied standard conformal inference to neural networks to 
construct prediction intervals with finite sample coverage guarantee. These
intervals are not specific to neural networks and do not leverage their 
ability to adapt to a wider class of distributions. We propose a new conformity 
score that is tailored to neural networks that fully exploits their 
flexibility. A related conformity score was simultaneously proposed by 
\citet{romano2019}.

In practice, it is often desirable to have a stronger coverage guarantee 
than average coverage. We propose a second neural-network-based PI with 
an approximate conditional coverage claim, where we condition on the observed
data. Both proposed procedures use a sample splitting 
strategy, where one part of the data is used to fit a network that outputs 
predictions and intervals while the other part is used to adjust these 
intervals to be valid PIs.

In Section \ref{sec:stats}, we formally set up the prediction problem and
explain in detail what we mean by valid PIs. This includes standard concepts
such as average coverage but also the new approximate conditional coverage 
claim which we term
Probably-Approximately-Valid (PAV). Section \ref{sec:PI-NN} introduces 
neural-network-based PIs
and the algorithms to produce valid prediction intervals. Extensive
simulations and real-data examples in Section \ref{sec:simulation}
highlight the improvements our methods make over competing algorithms:
Other methods either fail to provide adequate
coverage, compromise predictive accuracy, or fail to account for
heteroskedasticity. 

\subsection{Prediction Intervals} 
\label{sec:stats}

We consider a non-parametric regression setting, where $X$ denotes the
$\mathbb{R}^d$-valued covariate vector and $Y$ the $\mathbb{R}$-valued
response. The data set $D = (X_i,Y_i)_{1\leq i \leq n}$ contains $n$ i.i.d.
copies of the random variable $(X,Y)$. Throughout the paper, we will 
write $(X,Y)$ as new random variables that are independent of our data
$D$.

In general, a prediction interval $\Gamma_\alpha(X) = \Gamma_{D,\alpha}(X)$ is 
an interval-valued function of $X$, the data $D$, and the confidence level 
$\alpha$, such that, loosely speaking, a new observation falls within the 
interval with probability at least $1-\alpha$. This loose definition can be made 
precise in various ways. 
\begin{definition}
  The prediction interval $\Gamma_\alpha(X)$ controls average coverage if
  \begin{equation}
    \mathbb{P}(Y \in \Gamma_\alpha(X)) \geq 1-\alpha.
    \label{eqn:aveCov}
  \end{equation}
\end{definition}
Note that the probability in equation (\ref{eqn:aveCov}) is over all of the
random variables included: the new observation ($X,Y$) \emph{and the data} $D$.
In Section \ref{sec:PI_conformal}, 
we propose a neural-network-based inference procedure that controls average 
coverage.

One may object to average coverage because one would like a coverage guarantee
given a particular data set instead of averaging over all potential data sets.
Our second prediction interval method takes steps toward alleviating this
concern. We define a notion of prediction interval validity called Probably
Approximately Valid (PAV), which is inspired by the theory of probably
approximately correct (PAC) learning \citep{valiant1984}. A task is
PAC-learnable if, loosely speaking, regardless of the data generating
distribution, one can approximate the task arbitrarily well with high
probability given sufficient data. Similarly, a PAV interval provides a 
conditional coverage guarantee with high probability regardless of the data 
generating distribution. To this end, let $p(\Gamma_{\alpha}|D)$ denote the 
conditional probability that $\Gamma_{\alpha}(X)$ contains $Y$ conditional on
$D$, i.e.,
\begin{equation*}
  p(\Gamma_{\alpha}|D)=\mathbb{P}(Y \in \Gamma_{\alpha}(X) | D).
\end{equation*}
\begin{definition}
\label{def:pav}
  The prediction interval $\Gamma_\alpha(X)$ is \emph{Probably Approximately
  Valid (PAV)} if for all $\epsilon > 0$, all $\delta >0$ and all
  $n \geq n_0(\epsilon, \delta)$,
\begin{equation}
  \mathbb{P}\left(p(\Gamma_{\alpha}|D) \geq 1-\alpha - \epsilon\right) \geq 1- 
\delta. \label{eqn:pav}
\end{equation}
\end{definition}
This means, $\Gamma_{\alpha}(X)$ is PAV if the conditional probability of
$\Gamma_{\alpha}(X)$ covering $Y$ given $D$ is at least $1-\alpha-\epsilon$ for
all but $\delta100$\% of data sets $D$. 
In Section \ref{sec:PI_selection}, we provide a greedy algorithm that selects 
a neural-network-based PI that is PAV such that
$n_0(\epsilon, \delta)$ is of order $\minus\log(\delta)/\epsilon^2$.
Corollary \ref{cor:pav} shows that PAV prediction intervals can also 
control average coverage with proper choice of $\alpha$, $\epsilon$, and $\delta$.

Both proposed PIs are the result of a prediction-interval-specific neural network 
that outputs 
a three-dimensional vector and optimizes a loss function used in quantile 
regression. The next subsection describes and motivates this loss function.

\subsection{The PI-specific loss function} \label{sec:loss}

Let $N:\mathbb{R}^d \to \mathbb{R}^3$ be a network such that 
$N(x)=(l(x),m(x),u(x))$, where $l,m,u:\mathbb{R}^d \to \mathbb{R}$ with the 
restriction that $l(x) \leq m(x) \leq u(x)$ for all $x \in \mathbb{R}^d$. We use 
$m(x)$ to estimate the median of $Y$ given $X$, while $l(x)$ and $u(x)$ estimate 
the lower and upper bounds of our PI, respectively. The 
monotonicity, $l(x) \leq m(x) \leq u(x)$, is easily enforced by modifying any 
network that outputs a triple ($z_1,z_2,z_3$) to output ($z_1', z_2', z_3'$) 
given by $z_1' = z_1$, $z_2' = z_1' + \text{ReLU}(z_2-z_1')$, and $z_3' = z_2' + 
\text{ReLU}(z_3-z_2')$. Here, ReLU$(\cdot)$ is the rectified linear unit, 
$\max(0, \cdot)$. 

For $\tau \in [0,1]$ and $u \in \mathbb{R}$, let
$h_{\tau}(u) = (\tau-\mathds{1}_{u \leq 0})u$, be the standard loss function to 
estimate the $\tau$th quantile. We define the level-$\tau$ 
loss function evaluated on $N$ at $(x,y)$ by
\begin{IEEEeqnarray}{rCl} \label{eqn:PI-loss} 
  \mathcal{L}_\tau(N(x),y) & = & h_{\tau/2}(y-l(x)) + h_{1/2}(y-m(x))\\
    & & + \> h_{1-\tau/2}(y-u(x)).\nonumber
\end{IEEEeqnarray}
Letting $|D'|$ denote the cardinality of a set $D'$, the empirical 
risk of network $N$ on $D' \subseteq D$ is
\begin{equation} \label{empirical_loss}
  \mathcal{R}_{D', \tau}(N) = \frac{1}{|D'|} \sum_{(X_i,Y_i) \in D'} 
\mathcal{L}_\tau(N(X_i),Y_i).
\end{equation}
\begin{definition} \label{def:pi-net}
Denote the neural network $N_{D', \tau}(x)=(l_{D', \tau}(x),m_{D', 
\tau}(x),u_{D', \tau}(x))$ 
to be one that is fit on $D' \subseteq D$ by minimizing the 
empirical risk $\mathcal{R}_{D', \tau}(N)$. In the trivial case where $\tau=0$, 
we set $l_{0}(x)=-\infty$ and $u_{0}(x)=\infty$ for all $x \in \mathbb{R}^d$.
\end{definition}

By the strong law of large numbers, the empirical risk
$\mathcal{R}_{D, \tau}(N)$ converges almost surely to
$\mathbb{E}[\mathcal{L}_\tau(N(X),Y)]$ as $|D| \to \infty$. Let $q_{\tau}(x)
= \inf\{y: \mathbb{P}(Y \leq y ~ | ~ X=x) \geq \tau\}$ be the conditional
$\tau$-quantile of $Y$ given $X=x$, e.g., $q_{1/2}(x)$ is the
conditional median of $Y$ given $X=x$. Following standard texts 
\citep{Koenker05}, the triple $(q_{\tau/2}(x), q_{1/2}(x), q_{1-\tau/2}(x))$
is the minimizer of $\mathbb{E}[\mathcal{L}_\tau(N(X),Y)]$. For a given
confidence level $\alpha \in (0,1)$, $(q_{\alpha/2}(x), q_{1/2}(x),
q_{1-\alpha/2}(x))$ has the desirable properties that
\begin{equation}
  q_{1/2}(\cdot) ~ = ~ \underset{f:\mathbb{R}^d \to \mathbb{R}}{\arg\min} ~ 
\mathbb{E}|Y-f(X)| \label{eqn:prop1}\\
\end{equation}
and
\begin{equation}
  1-\alpha ~ = ~
  \mathbb{P}(q_{\alpha/2}(X) \leq Y \leq q_{1-\alpha/2}(X)) ~ | ~ X)
  \label{eqn:prop2}
\end{equation}
under minimal assumptions on the joint distribution of $(X,Y)$. 
If the problem is constrained to linear regression or M-estimation,
then the estimators resulting from empirical risk minimization are consistent
\citep{Wooldridge01}; however, under our minimal assumptions, it is not known
whether $N_{D, \tau}(x)$ consistently estimates $(q_{\alpha/2}(x), q_{1/2}(x),
q_{1-\alpha}(x))$. As the network $N_{D,\alpha}(x)$ does not generally provide 
the desired properties in finite samples, Section \ref{sec:PI-NN} provides 
two modifications with finite sample coverage guarantees based on sample 
splitting.

\subsection{Comments on the network architecture}

In the definition of the network $N_{D, \tau}(x)$, we only specified the output 
layer of the neural network and the loss function it is minimizing. The 
remaining network architecture can be chosen by the users in order to achieve 
their predictive goals. The term
architecture comprises the network design (e.g. number and depth of hidden
layers or dropout layers) and training parameters (e.g. number of epochs, batch
size, or regularization parameters). See \citet{goodfellow2016} and citations
therein for an overview on network architecture, regularization methods, and
optimization algorithms. In fact, one can use a pre-trained neural network
in combination with our proposed output layer and loss function to fit a 
neural network with accurate predictive performance and valid 
prediction intervals. See Section \ref{sec:simulation} for an application
to a real image dataset.

We advocate the use of a single network that outputs both the prediction and 
the prediction interval to ensure that the prediction is always contained in 
the interval. Crossing quantiles is a well-known problem in the 
literature (see \citet{xuming1997}) 
that can be easily be circumvented in this way without any loss in predictive 
accuracy (see Section \ref{sec:simulation}).

In fact, even our proposed loss function can be modified to the users needs.
For example, the midpoint estimate
is not restricted to be the median; one can also use the mean squared error 
instead of the absolute error in the loss function in order to estimate the 
conditional mean. If the underlying distribution is highly skewed, however, the 
conditional mean is not guaranteed to be lie between $q_{\alpha/2}(x)$ and 
$q_{1-\alpha/2}(x)$. We note that the mean squared error of the point 
estimate can be written as $h_{1/2}(y-m(x))^2$, so that squaring the other terms 
of equation (\ref{eqn:PI-loss}) puts them on the same scale. 

The supplemental materials demonstrate the validity of our methods over various 
architectures and loss functions. As this is guaranteed by our theorems, these 
examples are not included in the paper.


\section{Construction of valid PIs with neural networks} \label{sec:PI-NN}

Let $D_1$ and $D_2$ be an arbitrary partition
of $D$ into two disjoint sets. $D_1$ is used to select and train the network
and $D_2$ is used to adjust the resulting intervals to provide coverage
guarantees. We emphasize that all results are still valid for any
data-dependent architecture, as long as the dependency is only on $D_1$. As all 
networks in this paper are fit using $D_1$, we simplify our notation,
setting $N_{\tau}(x) = N_{D_1, \tau}(x)$ and analogously for $l_{\tau}(x)$,
$m_{\tau}(x)$ and $u_{\tau}(x)$.

\subsection{Average Coverage} 
\label{sec:PI_conformal}

To achieve average coverage, we use methods derived from conformal inference
\citep{Vovk05, Lei+18}. In general, conformal prediction intervals require a
fixed prediction procedure that is refit on an augmented data set. As it would
clearly be infeasible to refit a large network many times, this process can be
simplified by using classical sample slitting. \citet{Lei+18} refer to such
methods as split-conformal. 

For a given $x\in \mathbb{R}^d$ and $y \in \mathbb{R}$, the interval
$[l_{\tau}(x),u_{\tau}(x)]$ may not contain $y$; however, with $c$ given by
\begin{equation*}
c = \max\left( \frac{m_\tau(x) - y}{m_\tau(x) - l_\tau(x)},
  \frac{y - m_\tau(x)}{u_\tau(x) - m_\tau(x)} \right),
\end{equation*}
$y$ is an endpoint of the interval $\Gamma_\tau^c(x) = [m_\tau(x) - 
c(m_\tau(x)-l_\tau(x)), m_\tau(x)
+ c(u_\tau(x)-m_\tau(x))]$, and hence is contained in $\Gamma_\tau^c(x)$. These 
facts can be used in a conformal inference procedure to calibrate the 
PI. In essence, a constant $\hat{c} \in (0, \infty)$ is chosen such that at least 
(1-$\alpha$)100\% of the observations \emph{in the hold-out data}, $D_2$, are 
contained in the interval $\Gamma_\tau^{\hat{c}}(X)$. Details are given in 
Algorithm \ref{alg:splitConformal}. 

\begin{algorithm}[t]
\caption{Split Conformal Prediction Intervals}
\label{alg:splitConformal}
\begin{algorithmic}
  \STATE {\bfseries Input:} Holdout data $D_2$, network $N_{\tau}(x)$.
  \STATE {\bfseries Output:} Scaled network $N_{\tau}^{\hat{c}}(x)$ with 
  parameter $\hat{c}$.
  \STATE {\bfseries Set:}
      $c_i = \max\left( \frac{m_{\tau}(X_i) - Y_i}{m_{\tau}(X_i) - 
l_{\tau}(X_i)}, \frac{Y_i - m_{\tau}(X_i)}{u_{\tau}(X_i) - m_{\tau}(X_i)} 
\right)$ for \\\pushcode all $(X_i,Y_i) \in D_2$.
 \STATE {\bfseries Set:} $\hat c$ = $c_{(k)}$, $k = 
\left\lceil(1-\alpha)(|D_2|+1)\right\rceil$ and $c_{(k)}$ \\
  \pushcode the $k$th order statistic.
 \STATE {\bfseries Set:} $l_{\tau}^{\hat{c}}(x) = m_{\tau}(x) - 
\hat{c}(m_{\tau}(x) - l_{\tau}(x))$ and \\
  \pushcode $u_{\tau}^{\hat{c}}(x) = m_{\tau}(x) + \hat{c}(u_{\tau}(x) - 
m_{\tau}(x))$.
 \STATE {\bfseries Return:} $N_{\tau}^{\hat{c}}(x) = (l_{\tau}^{\hat{c}}(x), 
m_{\tau}(x) ,u_{\tau}^{\hat{c}}(x))$.
\end{algorithmic}
\end{algorithm}

\begin{theorem}
\label{thm:conformal}
Let $N_{\tau}(x) = N_{D_1,\tau}(x)$ be as in Definition \ref{def:pi-net} and
$N_{\tau}^{\hat c}(x)$ be the result of Algorithm \ref{alg:splitConformal}.
Set $\Gamma_{\alpha}^{\hat{c}}(x) =
[l_{\tau}^{\hat{c}}(x),u_{\tau}^{\hat{c}}(x)]$. Then
$\mathbb{P}\left(Y \in \Gamma_{\alpha}^{\hat{c}}(X) \right) \geq 1-\alpha.$
\end{theorem}

Note that Theorem \ref{thm:conformal} also holds for any data-dependent 
$\hat{\tau}$ as long as the dependence of $\hat{\tau}$ on the data is only 
through $D_1$. As $\tau$ controls the width of the estimated interval and that 
is precisely 
what $\hat c$ is selected to calibrate, we suggest setting $\tau = \alpha$ in 
practice. In the simulations of Section \ref{sec:simulation}, we 
typically observe $\hat{c}\geq1$ for $N_{\alpha}(x)$, but $\hat{c} 
\to 1$ as $|D_1|$ increases, suggesting that $N_{\alpha}(x)$ has 
asymptotic $1-\alpha$ average coverage.

\begin{proof}[Proof of Theorem \ref{thm:conformal}]
Because the network $N_{\tau}(x)$ is fit on $D_1$ which is independent of 
$D_2$, the statistics $c_i$ are i.i.d. conditional on $D_1$. For the new 
observation $(X,Y)$, set
\begin{equation*}
 c' = \max\left( \frac{m_{\tau}(X) - Y}{m_{\tau}(X) - l_{\tau}(X)}, \frac{Y -
m_{\tau}(X)}{u_{\tau}(X) - m_{\tau}(X)} \right).
\end{equation*}

Conditional on $D_1$, $c'$ is independent of $c_i$, has the same distribution
as $c_i$, and the rank of $c'$ among the $c_i$s is uniform over the set of
integers $\{1,2,\ldots,|D_2|+1\}$. Therefore, $\mathbb{P}(c' > c_{(k)} ~ | ~
D_1) \leq \alpha$, where $k = \left\lceil(1-\alpha)(|D_2|+1)\right\rceil$ and
$c_{(k)}$ is the $k$th order statistic of the $c_1,\dots,c_n$. But this implies 
$\mathbb{P}(Y \in \Gamma_{\alpha}^{\hat{c}}(X) ~ | ~ D_1) \geq 1-\alpha$.
\end{proof}

If ties among $c_i$s only happen on a set of measure 0, or if we use a random
tie-breaking rule, the probability of coverage can be upper bounded by
$1-\alpha+1/(|D_1|+1)$, meaning that the intervals
are not unnecessarily wide (cf. Theorem 2.2 of \citet{Lei+18}).
Algorithm \ref{alg:splitConformal} is a modified version of the
Split-Conformal algorithm of \citet{Lei+18}. Our statistics $c_i$ are tailored 
to our neural network which outputs an ordered triple, whereas previous 
algorithms use the residuals $R_i = |Y_i-m(X_i)|$ or standardized variants.

\subsection{PAV} \label{sec:PI_selection} 

In order to provide PAV intervals, we select $\hat\tau$ using $D_2$ so that
$l_{\hat\tau}(x)$ and $u_{\hat\tau}(x)$ adaptively estimate the quantiles
$q_{\alpha/2}(x)$ and $q_{1-\alpha/2}$. This adaptation is required as we are
not guaranteed that $l_\alpha(x)$ and $u_\alpha(x)$ accurately estimate these 
quantities.
Let $\hat{p}(N_{\tau}, D_2)$ be the empirical coverage probability of the
neural network $N_{\tau}$ on the data set $D_2$, i.e.,
\begin{equation*}
   \hat{p}(N_{\tau}, D_2) = \frac{1}{|D_2|} \sum_{(X_i,Y_i) \in D_2} 
\mathds{1}_{Y_i \in [l_\tau(X_i),u_\tau(X_i)]}.
\end{equation*} 
We choose a parameter $\hat\tau$ over a grid $G_K$, where
\begin{equation*}
  G_K = \{\tau_{(1)}, \dots, \tau_{(K)})\}, \quad 1 > \tau_{(1)} > \dots \tau_{(K)} > 0,
\end{equation*}
such that the network $N_{\hat\tau}(x)$ has $1-\alpha$
coverage on $D_2$. Details are given in Algorithm $\ref{alg:pav}$.

\begin{algorithm}
\caption{PAV Prediction Intervals}
\label{alg:pav}
\begin{algorithmic}
  \STATE {\bfseries Input:} Holdout data $D_2$, a grid $G_K$, Networks $N_{\tau}(x)$
  with $\tau \in G_K$.
  \STATE {\bfseries Output:} Network $N_{\hat \tau}(x)$.
  \STATE {\bfseries Set:} $\hat \tau = 0$.
  \STATE {\bfseries For} $i=1$ to $K$: \\
  \quad \bfseries{If} $\hat{p}(N_{\tau{(i)}}, D_2) \geq 1-\alpha$: \\
  \quad\quad \bfseries {Set:} $\hat \tau = \tau_{(i)}$. \\
  \quad\quad {\bfseries End.}
  \STATE {\bfseries Return:} $N_{\hat \tau}(x)$.
\end{algorithmic}
\end{algorithm}

\begin{theorem} \label{thm:pav} 
  Let $N_{\tau}(x) = N_{D_1,\tau}(x)$ be as in Definition \ref{def:pi-net} and
  $N_{\hat \tau}(x)$ be the result of Algorithm \ref{alg:pav}. Set
  $\Gamma_{\alpha}^{\hat{\tau}}(x) = [l_{\hat{\tau}}(x),u_{\hat{\tau}}(x)]$. Set
  $p(\Gamma_{\alpha}^{\hat{\tau}}|D) = \mathbb{P}(Y \in
  \Gamma_{\alpha}^{\hat{\tau}}(X) ~ | ~ D)$ and $n_2 = |D_2|$. Then,
  \begin{equation*}
    \mathbb{P}\left(p(\Gamma_{\alpha}^{\hat{\tau}}|D) \leq 1-\alpha-\epsilon
    \right) \leq K\exp(-2\epsilon^2 n_2).
  \end{equation*}
  
\end{theorem}
The theorem continues to hold for any data-dependent grid, $\hat G_K$, as long
as the dependence of $\hat G_K$ on the data is only
through $D_1$. Given the tendency of flexible models such as neural networks to
over-fit the data, $[l_\alpha(x), u_\alpha(x)]$ typically under-covers in
finite samples. As such, we suggest setting $G_K \subseteq [0,\alpha]$, and in
the simulations of Section \ref{sec:simulation} we typically observe
$\hat{\tau} \leq \alpha$ but that $\hat{\tau} \to \alpha$ from below as $|D_1|$
increases. Note that by solving the equation $\delta = K\exp(\minus2\epsilon^2 n_2)$, we
get that for all $n_2 \geq
n(\epsilon,\delta,K)=\minus\log(\delta/K)/(2\epsilon^2)$,
\begin{equation*}
   \mathbb{P}\left(p(\Gamma_{\alpha}^{\hat{\tau}}|D) \geq 1-\alpha-\epsilon
   \right) \geq 1-\delta.
\end{equation*}

\begin{proof}[Proof of Theorem \ref{thm:pav}]
 By definition of $N_{\hat{\tau}}(x)$, we have 
  \begin{IEEEeqnarray*}{rCl}
    \multicol{3}{l}{\mathbb{P}\left(p(\Gamma_{\alpha}^{\hat{\tau}}|D) 
     \leq 1-\alpha - \epsilon \right)}\\
    \quad & \leq & \mathbb{P}\left(p(\Gamma_{\alpha}^{\hat{\tau}}|D) \leq  
\hat{p}(N_{\hat{\tau}}, D_2) - \epsilon \right) \\
    & = & \mathbb{E} \left[ \mathbb{P}\left(p(\Gamma_{\alpha}^{\hat{\tau}}|D) 
\leq  \hat{p}(N_{\hat{\tau}}, D_2) - \epsilon  ~ | ~ D_1 \right) \right].
  \end{IEEEeqnarray*}
  We use the union bound to bound the conditional probability within the 
  expectation from above by
  \begin{equation*}
    \Scale[0.9]{
    \sum\limits_{\tau \in G_K\cup \{0\}} \mathbb{P}\left(p([l_{\tau}(X),u_{\tau}(X)] ~| ~ 
    D) \leq  \hat{p}(N_{\tau}, D_2) - \epsilon  ~ | ~ D_1 \right).
    }
  \end{equation*}
  For $\tau=0$, we have $[l_{\tau}(X),u_{\tau}(X)]=[-\infty,\infty]$ and the
  corresponding
  summand in the previous expression is equal to $0$. For $\tau \neq 0$, we
  have $p([l_{\tau}(X),u_{\tau}(X)] ~| ~ D) = p([l_{\tau}(X),u_{\tau}(X)] ~| ~
  D_1)$, because the event in the first conditional probability is independent 
  of $D_2$. Observe that, conditional on $D_1$, $\hat{p}(N_{\tau}, D_2)$ is 
  the mean of $n_2$ Bernoulli-trails with mean
   $p([l_{\tau}(X),u_{\tau}(X)] ~| ~ D_1)$. By Hoeffding's inequality, the 
   corresponding summand in the previous expression is bounded by
  $\exp(\minus2 \epsilon^2 n_2)$.
\end{proof}

Even though the PAV interval $\Gamma_{\alpha}^{\hat{\tau}}(X)$ does not
generally provide average coverage, this can be achieved by defining a more 
conservative PAV interval.

\begin{corollary} \label{cor:pav}
  Fix $\epsilon>0$ such that $\alpha-\epsilon>0$. Let $\Gamma_{\alpha-\epsilon}^{\hat{\tau}}(x)$ be defined as in Theorem \ref{thm:pav} with $\alpha-\epsilon$ replacing $\alpha$. For all
  \begin{equation*}
    n_2 \geq \frac{-2\log(\epsilon/(2K(1-\alpha+\epsilon/2)))}{\epsilon^2}
  \end{equation*}
  we have
  \begin{equation*}
    \mathbb{P}\left(Y \in \Gamma_{\alpha-\epsilon}^{\hat{\tau}}(X) \right) \geq
1-\alpha.
  \end{equation*}
\end{corollary}
To prove the result observe that $\mathbb{P}(Y\in 
\Gamma_{\alpha-\epsilon}^{\hat\tau}(X)) \geq (1-\alpha + 
\epsilon/2)(1-K\exp(\minus\epsilon^2 n_2/2)).$

\subsection{Discussion of contributions}
\label{sec:contribution}


The split-conformal and adaptive selection approaches of the previous 
subsections admittedly constitute a revival of sample splitting. Practitioners 
have long used sample-splitting as a valid, tractable method for tuning or 
testing and can easily implement it. The methodology followed in this paper 
differs subtly but significantly from the standard use of sample splitting or 
conformal inference. Typically one creates either residuals or standardized 
residuals in order to estimate an additive adjustment factor for creating a PI 
of the form estimate $\pm$ error. 
It is clear that such intervals may not be appropriate for skewed or multimodal 
distributions. Instead of designing more complex standardizations for residuals, 
this paper leverages the power of NN in order to estimate the quantiles 
directly, using subsequent adjustments for finite sample guarantees. 

We conjecture that the desirable properties in equations 
(\ref{eqn:prop1}) and (\ref{eqn:prop2}) are satisfied in a well-defined 
asymptotic setting: First, neural networks are universal approximators (cf. 
\citet{cybenko1989,hornik1991}). Second, given a sufficiently large data set, 
and under the assumption that we can minimize the empirical risk, a neural 
network is able to learn the conditional mean of $Y$ given $X$ (cf. 
\citet{bauer2017}). This conjecture is supported by simulation evidence 
presented in the supplement where we observe that $N_{D, \tau}(x) \to 
(q_{\alpha/2}(x), q_{1/2}(x), q_{1-\alpha/2}(x))$ as $|D| \to \infty$.

We propose PIs with average coverage control and an approximate conditional 
coverage control (PAV). A stronger coverage claim would be to control the 
conditional coverage probability, conditional on the input $X$. However, 
this only achievable under much more restrictive assumptions. Also, we note that
our only (and crucial) assumption is i.i.d.-ness of the data. This means, our PIs 
are in general not valid for out-of-distribution samples.

\section{Simulations} \label{sec:simulation}

Throughout this section, we refer to
$\Gamma_{\alpha}^{\hat{c}}(X)$ from Section \ref{sec:PI_conformal} as
\textit{conf-nn} and $\Gamma_{\alpha}^{\hat{\tau}}(X)$ from Section
\ref{sec:PI_selection} as \textit{pav}. 
We compare our methods to 5 different procedures that use 
different types of loss functions that are designed to fit neural 
networks that output PIs. The first (\textit{conf-fw}) is the 
classical, ``fixed-width'' conformal method that uses the absolute 
residuals as conformity score to create valid PIs \citep{Pap11}. The second 
method (\textit{high-q}) is the so-called ``high-quality''
driven method of \citet{Pearce+18}. This method uses a loss 
function that penalizes in-sample mis-coverage and interval length. The third 
method (\textit{neg-ll}) uses the
Gaussian negative-log-likelihood to estimate the mean and variance
of the target \citep{nix1994}. 
The fourth method is the calibrated regression model of \citet{kuleshov2018},
where the Gaussian negative-log-likelihood is used to estimate the 
distribution and and isotonic regression for calibration. Because of
its probabilistic approach, we refer to this method as \textit{bayes}.
Finally, we also consider the interval $\Gamma_{\alpha}(X)$ (\textit{qreg-un}), 
i.e., the unadjusted interval function of the network $N_\alpha(X)$. 

In all data examples, we split the data $D$ into a training set $D_1$, a 
validation set
$D_2$, and a test set $D_3$. $D_1$ is used to train the network, $D_2$ is used 
to calibrate the intervals for \textit{pav}, \textit{conf-nn}, \textit{conf-fw}
and \textit{bayes}. Because $\textit{neg-ll}$ and $\textit{high-q}$ were
very sensitive to the choice of hyperparameters in our experiments, we used $D_1$ 
to train these networks and $D_2$ for hyperparameter selection. 
All data experiments set $\alpha=.1$ and 
calculate results using $D_3$, which was unseen by the models. The $\hat{\tau}$ 
in $\Gamma_{\alpha}^{\hat{\tau}}(X)$ is selected on the grid 
$G_{10}=\{.1,.09,\ldots,.01\}$ and we set $\tau=\alpha$ when constructing 
$\Gamma_{\alpha}^{\hat{c}}(X)$. In all experiments, we used the Adam optimizer 
with learning rate 0.01 and no learning rate decay. The number of epochs was 
chosen by cross-validation on $D_1$. 
In each data example, the same network architecture is used for all methods and 
is trained in Python using the PyTorch library \citep{paszke2017}. These 
architectures are described along with the data examples. All 
computations were performed on a Google-Cloud-Platform instance with a NVIDIA 
Tesla K80 GPU. 



\subsection{Artificial data}

We simulate covariates $X\in[0,1]^{100}$, where each entry is drawn i.i.d. from 
a standard uniform distribution.
The response is given by
\begin{equation*}
  Y = f(\beta'X) + \epsilon, \quad \epsilon \sim N(0,1 + (\beta'X)^2),
\end{equation*}
where $f(x) = 2\sin(\pi x) + \pi x$ and only the first 5 components of
$\beta\in\{0,1\}^{100}$ are non-zero and equal to 1. This is a challenging 
setting because the model is sparse, non-linear in $X$, and heteroskedastic in 
$Y$ given $X$. We compare all the methods to an oracle that knows this 
data-generating process and thus the true conditional median, $q_{1/2}(X) = 
f(\beta'X)$, as well as the uniformly most accurate, 
unbiased PI, $f(\beta'X) \pm z_{\alpha/2}\sqrt{1 + (\beta'X)^2}$,
where $z_{\alpha/2}$ is the $(1-\alpha/2)$-quantile of the standard
normal distribution. We generate $|D_1 \cup D_2| = 100,000$ and used $3/4$ of 
the data for training and $1/4$ for validation. In the supplement, we provide 
results for data set sizes ranging between 5,000 and 100,000 with comparable 
results. We repeated the experiment 10 times. For each method, we 
train a neural network with one hidden layer and 200 hidden nodes using 80-100 
passes through the data. An independent data set of size 20,000 is used for 
testing.

\begin{figure*}[!th]
\centerline{\includegraphics[trim=0 60 0 50,
  clip, width=\textwidth]{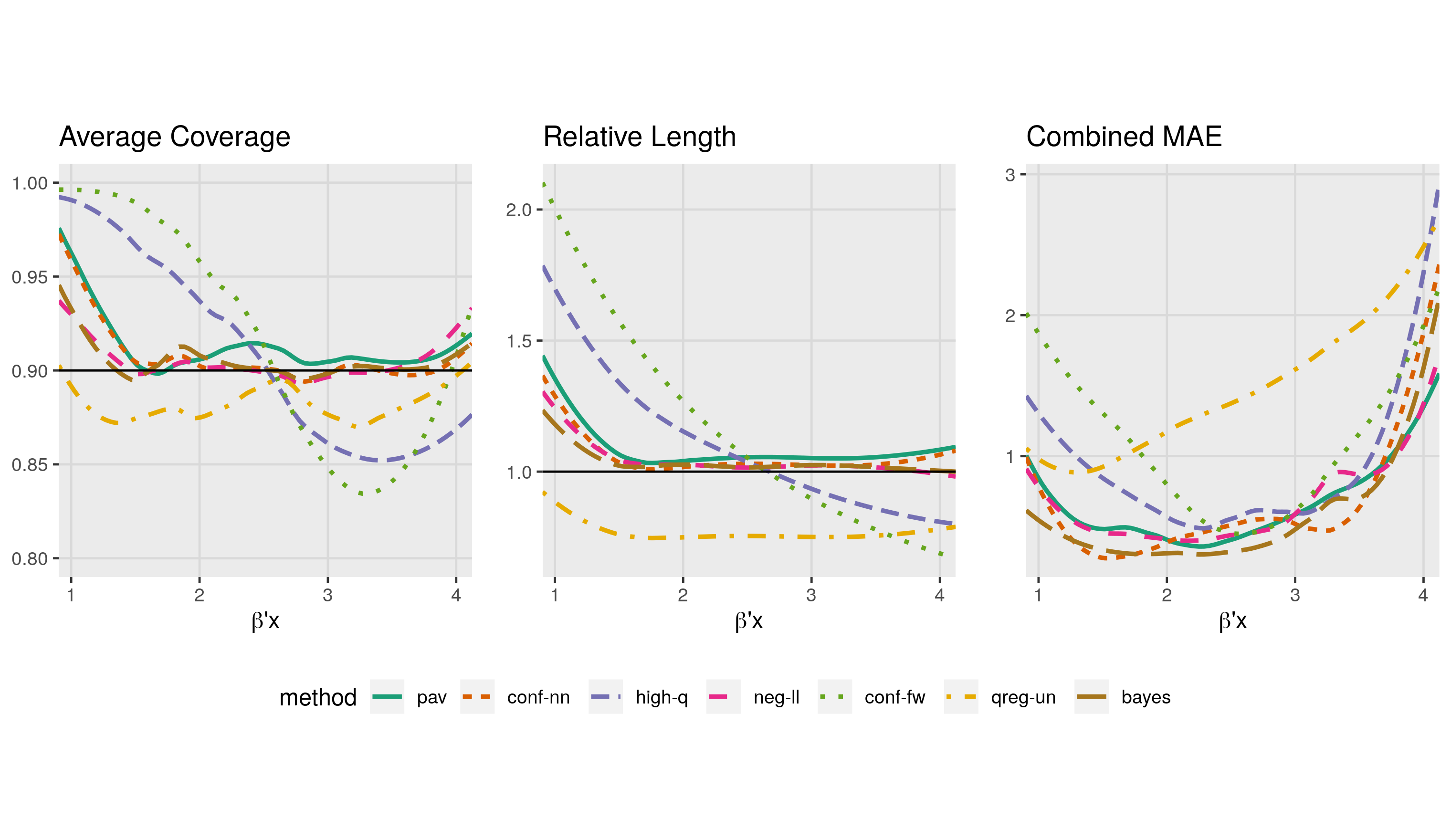}}
 \caption{Results on Simulated Data. Each measure is plotted as a 
function of the true linear component, $\beta'X$. The combined MAE is the 
mean absolute error for estimating the oracle: ($q_{1/2}(X)$, 
$q_{\alpha/2}(X)$, $q_{1-\alpha/2}(X)$).}
\end{figure*} 
 \label{fig:simulation}

Simulation results are summarized in Figure \ref{fig:simulation}, which shows 
the empirical coverage, the interval length relative to the oracle, and the
mean absolute error (MAE) from estimating ($q_{1/2}(X)$, $q_{\alpha/2}(X)$, 
$q_{1-\alpha/2}(X)$). Each measure is plotted as a function of $\beta'X$ 
in order to observe how various methods adapt to heteroskedasticity. Both 
\textit{conf-nn} and \textit{pav} provide coverage close to the nominal level 
throughout the majority of the domain of $\beta'X$. Only for small $\beta'X$, the 
two methods are too wide. This can be explained by the fact that in this area 
only a few observations are available (see the data generation process) to 
estimate the conditional quantiles. This demonstrates that both 
procedures are accurately estimating the true underlying distribution on a wide
range of the input space. \textit{bayes} and \textit{neg-ll} perform equally 
well, which is not surprising 
because both procedures minimize the Gaussian negative-log-likelihood. In terms 
of length and estimation error, \textit{conf-nn}, \textit{pav}, \textit{bayes} 
and \textit{neg-ll} all provide quality estimates for a wide range of 
$\beta'X$. 

\textit{conf-fw} and \textit{high-q} on the other hand over-cover in the left 
tail and undercover in the right tail as they do not capture the 
heteroscedasticity of the true distribution. This is observed in their lengths 
relative to the oracle. We see that \textit{qreg-un} undercovers through the 
entire input space, which confirms the need for calibration.

\subsection{Real Data}
\label{sec:data}

In each of the real data examples, the ratio of observations in
$D_1$, $D_2$ and $D_3$ was equal to 3:1:1. Each experiment was repeated 20
times, i.e., the data set is randomly partitioned into three parts, then the 
model is trained on $D_1$, tuned on $D_2$, and evaluated on $D_3$. This 
process constitutes one repetition. We discuss the results from four real 
datasets here, about which further information is given in the supplement. When 
applicable, covariates were standardized to have mean 0 and variance 1.

We consider two datasets from the data science platform Kaggle and two from the 
UCI machine learning repository \citep{UCI}. The first is a bike share dataset 
from UCI that contains the hourly count of bike rentals in the Capital bikeshare 
system along with weather information between 2011 and 2012 \citep{BikeShare}. 
The task is to predict the number of bike rentals based on weather information. 
The other UCI dataset provides hourly traffic volume, westbound on 
I-94, in Minneapolis-St Paul, Minnesota between 2012 and 2018. The dataset also 
includes weather and holiday information. For both datasets, we trained a 
network with one hidden layer and 100 hidden nodes. 

From Kaggle, we consider an image data set and a standard regression 
task.
The image dataset contains 12,611 observations, each consisting of an x-ray of a
patient's hand. The task is to predict the patient's age using 
the x-ray. We use an intermediate layer of the pre-trained Inception V3 network 
as the feature extractor \citep{szegedy2016}. We trained a neural network on 
the extracted features with one hidden layer and $300$ hidden nodes. We used 
data augmentation (random rotation and horizontal flips of the images) on $D_1$ 
to reduce over-fitting. The regression dataset consists of 21,613 sale prices 
for homes in King County, Washington, between May 2014 and May 2015. There are 
19 covariates describing the features of the house which are used to predict 
log sale price. For each method we trained a neural network with one hidden 
layer, with 100 hidden nodes, and between 80 and 200 passes through the data.

\begin{figure*}[!htp]
\centerline{\includegraphics[trim=0 70 0 70, clip, width=1\textwidth]{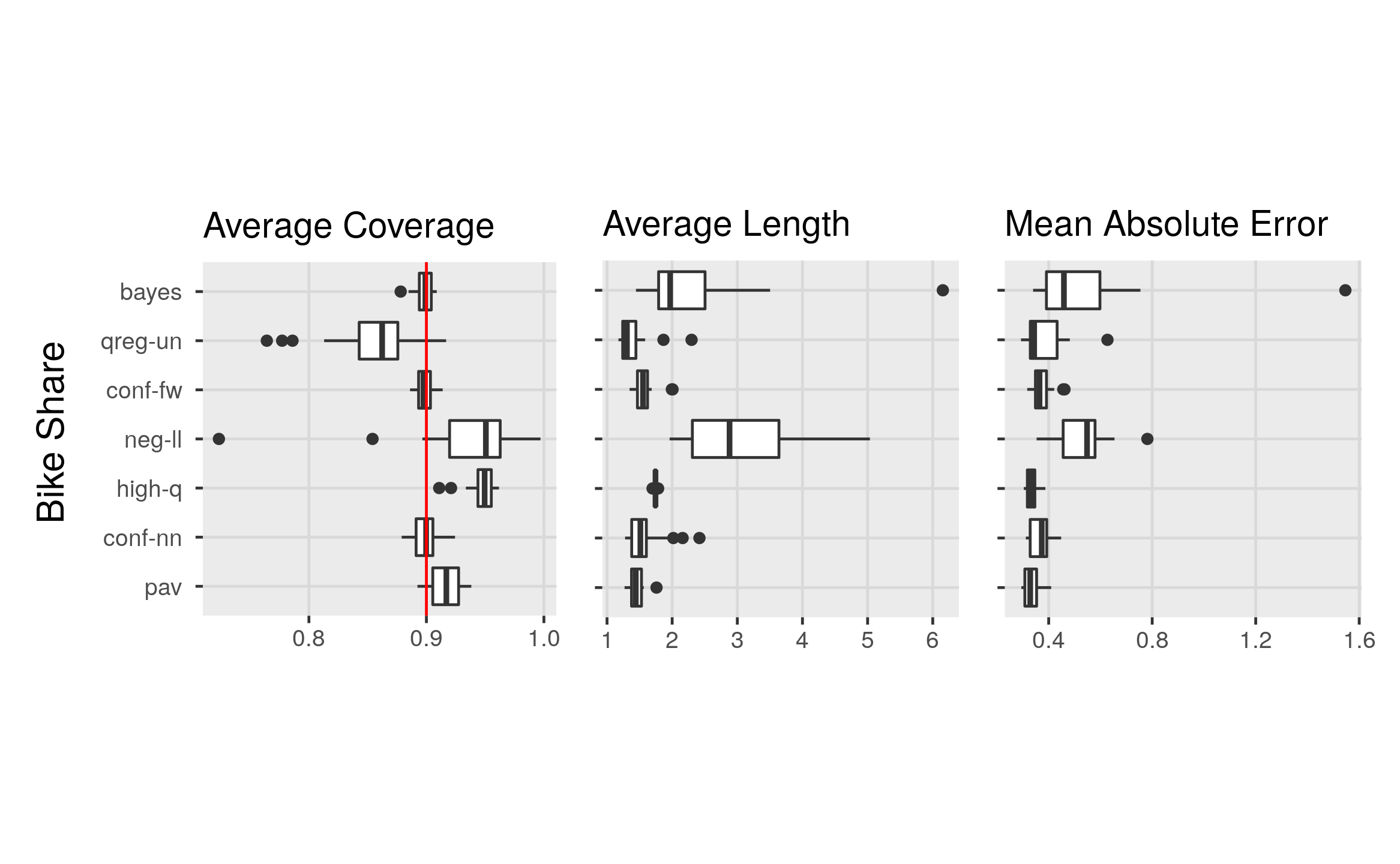}}
\centerline{\includegraphics[trim=0 70 0 95, clip, width=1\textwidth]{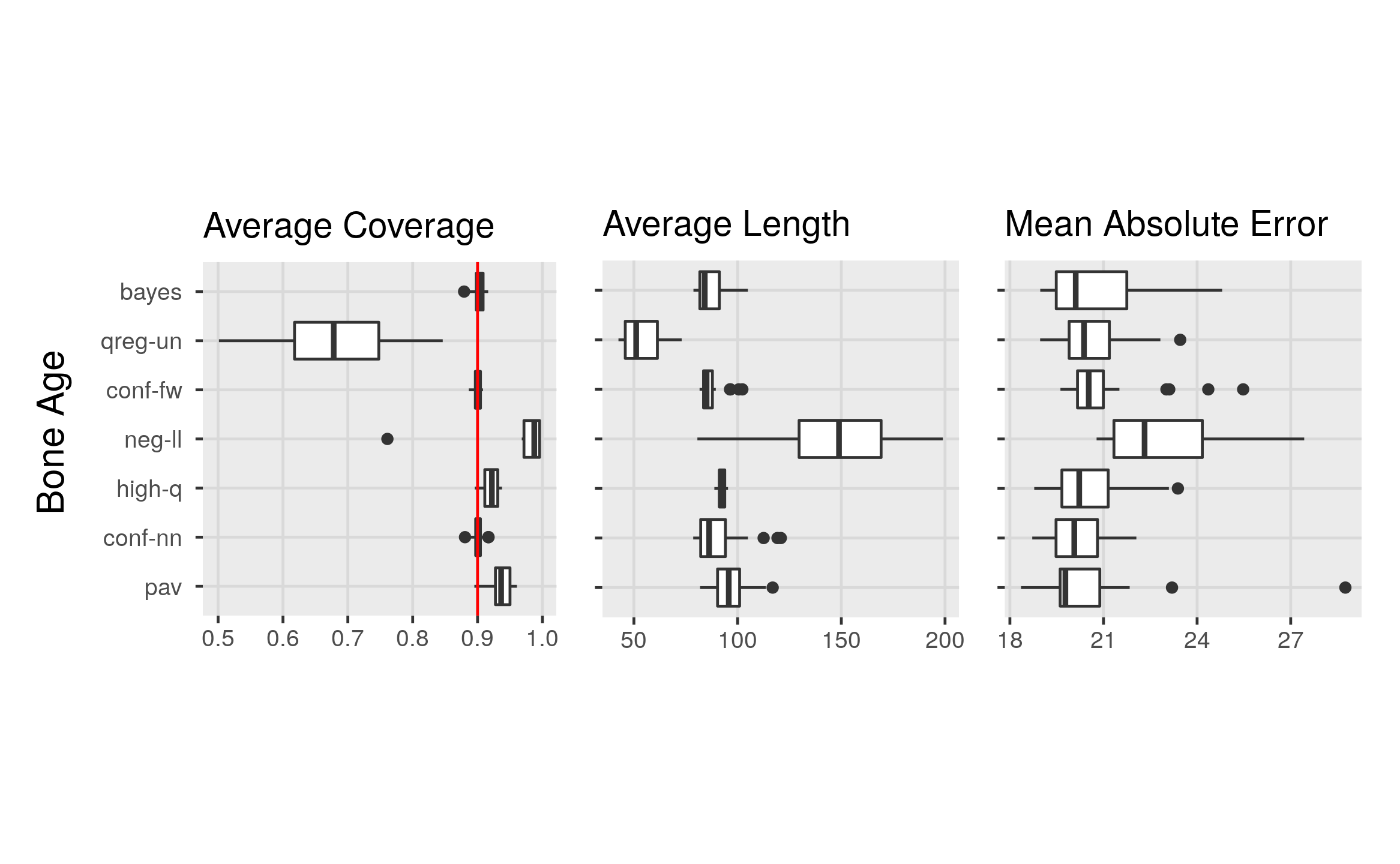}}
\centerline{\includegraphics[trim=0 70 0 95, clip, width=1\textwidth]{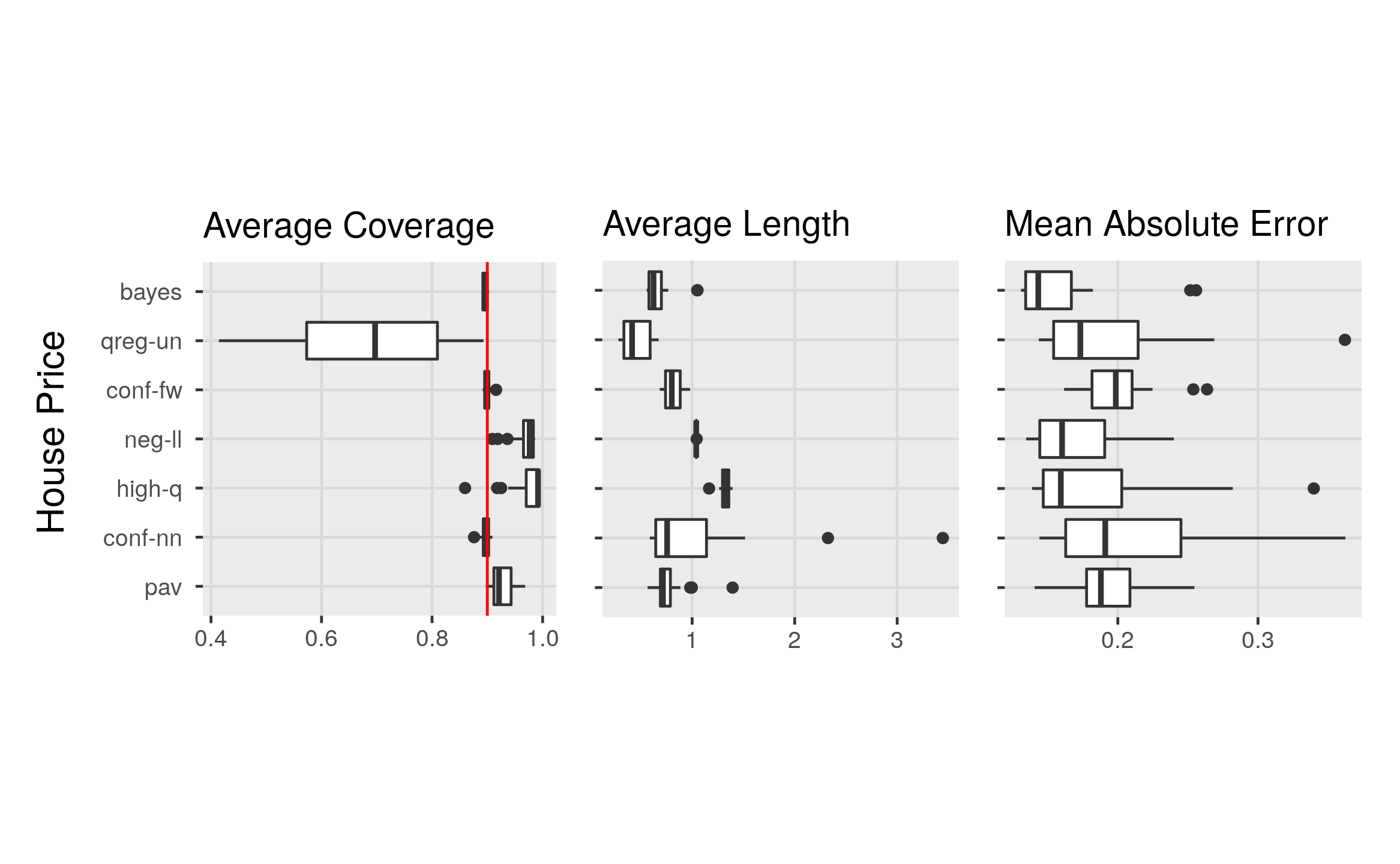}}
\centerline{\includegraphics[trim=0 70 0 95, clip, width=1\textwidth]{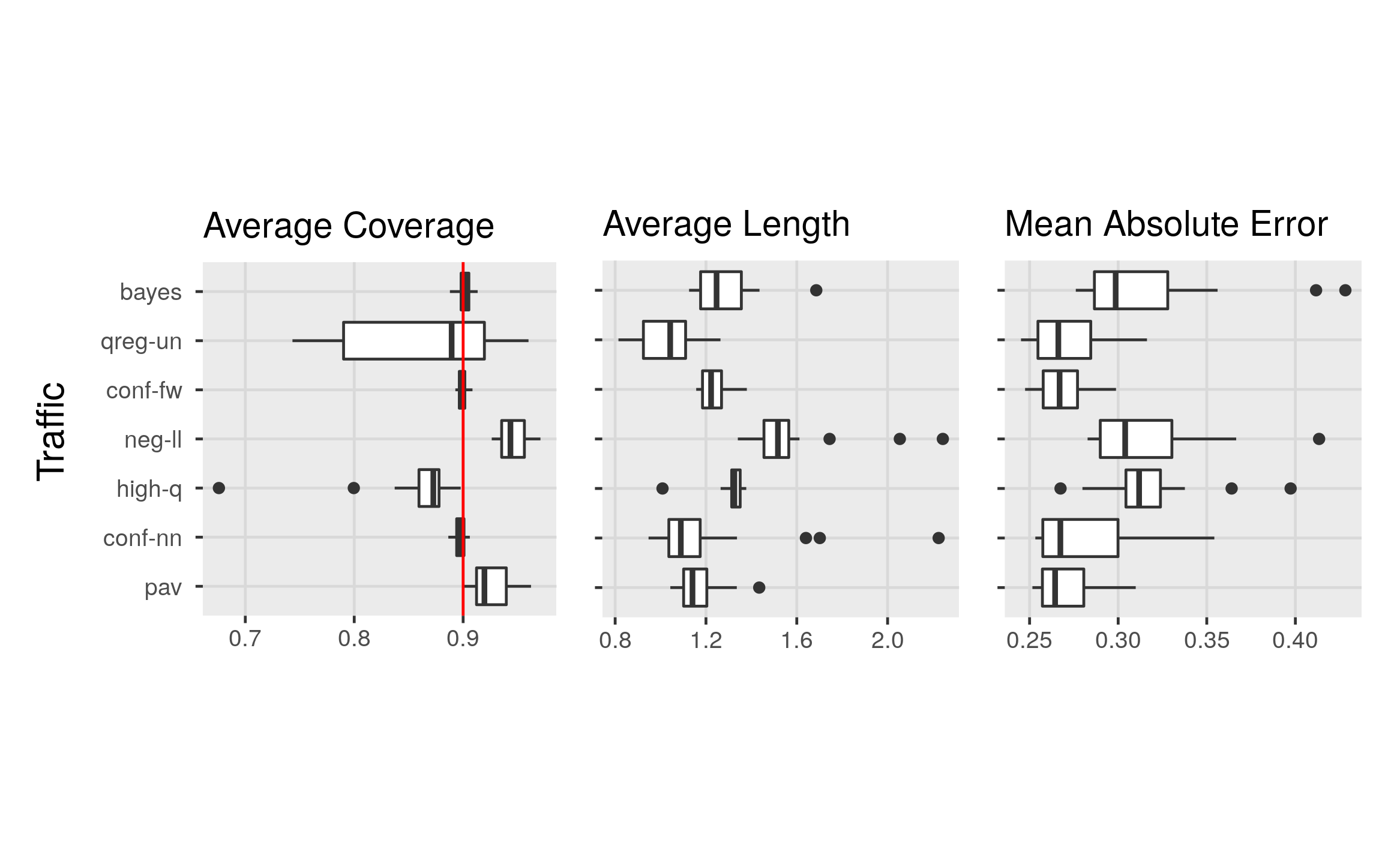}}
 \caption{Summary of results on real data. Each row corresponds to a different 
data set and are organized alphabetically. The red line in the left-most plots 
denote the nominal confidence level.}
 \label{fig:realData}
\end{figure*}

Figure \ref{fig:realData} summarizes the results on these data. It shows the 
empirical coverage, the average length, and the mean absolute error (MAE) of 
each method for predicting the target. Starting with average coverage, we see 
that \textit{conf-nn} and \textit{pav} provide average coverage at nearly exact 
nominal level for all repetitions and all datasets as guaranteed by our 
theorems. We can see that \textit{pav} is more conservative, due to its stronger 
guarantee, though potentially also due to the coarse selection grid for $\tau$. 
The unadjusted method, \textit{qreg-un}, does not provide coverage, which 
confirms our presumption that the neural network $N_\alpha(X)$ fails to 
accurately estimate the conditional quantiles in finite samples. This also 
strongly affirms the necessity to adjust neural network based intervals. 
\textit{bayes} also provides precise coverage and the remaining methods primarily 
over-cover, partitulcarly \textit{neg-ll}. We see that \textit{high-q} does not 
provide consistent coverage on the Traffic data and found it to be 
highly-sensitive to hyperparameters.

Among the valid methods, \textit{conf-nn} and \textit{pav} are always among the
shortest intervals. \textit{bayes} performs significantly worse on the Bike Share and 
traffic datasets, with many repetitions having double or triple the length and 
error of \textit{conf-nn} and \textit{pav}. While \textit{conf-fw} 
provides surprisingly short intervals on average, we note that in the presence 
of heteroscedasticity it can occur that inferior fixed-width intervals are
narrower, on average, than the optimal intervals. \textit{high-q} is slightly 
wider and \textit{neg-ll} is considerably wider on average, also showing they 
they can be too conservative.

In terms of the estimation error \textit{conf-nn}, \textit{pav} and 
\textit{conf-fw} have similar performance on all datasets. We emphasize this as 
\textit{conf-fw} fits a neural network that minimizes mean absolute error 
directly, while \textit{conf-nn} and \textit{pav} both use the loss function 
given (\ref{eqn:PI-loss}). This demonstrates that our proposed loss function 
does not comprise the predictive accuracy of the fitted networks.

\section{Conclusion}

In this paper we demonstrated how to use standard statistical 
techniques such as sample splitting and quantile regression to fit a 
neural network that outputs accurate predictions and valid prediction intervals. 
We proposed two procedures with provable coverage guarantees that improve
current state-of-the-art methods in terms of interval length and predictive 
accuracy. The ease and transparency of the two procedures advocate their 
application to many deep neural networks in order to express the 
uncertainty of their predictions.
As data sets grow in size, the cost of using a 
validation set decreases, making the observation of \citet{Barnard74} even more 
accurate today:
\begin{quote}
   The simple idea of splitting a sample into two and then developing the 
hypothesis on the basis of one part and testing it on the remainder may perhaps 
be said to be one of the most seriously neglected ideas in statistics, if we 
measure the degree of neglect by the ratio of the number of cases where a method 
could give help to the number of cases where it is actually used.
\end{quote}

\newpage
\bibliography{bibliography.bib}

\end{document}